%% file: main.tex
\documentclass[letterpaper]{article} %
\usepackage{style}
\usepackage{times}  %
\usepackage{helvet}  %
\usepackage{courier}  %
\usepackage{graphicx} %
\usepackage[square,numbers]{natbib}
\usepackage{caption} %
\usepackage{algorithm}
\usepackage{algorithmic}

\usepackage{newfloat}
\usepackage{listings}
\DeclareCaptionStyle{ruled}{labelfont=normalfont,labelsep=colon,strut=off} %
\floatstyle{ruled}
\newfloat{listing}{tb}{lst}{}
\floatname{listing}{Listing}
\usepackage[utf8]{inputenc} %
\usepackage[T1]{fontenc}    %

\usepackage{url}            %
\usepackage{booktabs}       %
\usepackage{amsfonts}       %
\usepackage{nicefrac}       %
\usepackage{microtype}      %
\usepackage{xcolor} %
\usepackage{float}
\floatstyle{plaintop}
\restylefloat{table}
\usepackage{caption}
\usepackage{placeins}
\usepackage{dsfont}

\usepackage[page]{appendix} %

\usepackage{xcolor}
\usepackage{mathtools,amsfonts,amssymb,mdframed,xspace,xfrac,multicol,bm}
\usepackage{makecell}
\usepackage{algorithm,algorithmic}
\usepackage{eucal}
\usepackage{amssymb,amsmath,amsopn,mathtools}
\usepackage{mathrsfs}
\usepackage{complexity}
\usepackage{enumitem}
\usepackage{subcaption}
\usepackage{dsfont}
\usepackage{xr}

\usepackage{framed}
\usepackage[capitalize,noabbrev]{cleveref}

\renewcommand{\E}{\mathbb{E}}

\renewcommand{\R}{\mathbb{R}}

\newcommand{\mc}[1]{\ensuremath{\mathcal{#1}}\xspace}
\newcommand{\mb}[1]{\ensuremath{\mathbf{#1}}\xspace}

\newcommand{\bX}{{\mb{X}}}

\newcommand{\bP}{{\mb{P}}}

\newtheorem{theorem}{Theorem}[section]
\newtheorem{lemma}[theorem]{Lemma}
\newtheorem{definition}[theorem]{Definition}

\title{Towards Fair and Calibrated Models}
\author{
    Anand Brahmbhatt\textsuperscript{\rm 1} \\ Google Research India \\ {\tt anandpareshb@google.com} \and
    Vipul Rathore \\ IIT Delhi \\ {\tt rathorevipul28@gmail.com} \and
    Mausam \\ IIT Delhi \\ {\tt mausam@iitd.ac.in} \and
    Parag Singla \\ IIT Delhi \\ {\tt parags@iitd.ac.in}
}

\usepackage{bibentry}

\makeatletter

\newcommand\blfootnote[1]{%
  \begingroup
  \renewcommand\thefootnote{}\footnote{#1}%
  \addtocounter{footnote}{-1}%
  \endgroup
}

\begin{document}

\maketitle
\blfootnote{[1] The majority of the work was completed during association with IIT Delhi.}

\input{abstract}

\input{introduction}

\input{related_work}

\input{theory}

\input{joint_optimization_techniques}

\input{empirical_evaluation}

\input{conclusion}

\FloatBarrier

\bibliographystyle{plainnat}
\bibliography{references}

\clearpage
\appendix

\input{appendix}

\end{document}

%% file: abstract.tex
\begin{abstract}
    Recent literature has seen a significant focus on building machine learning models with specific properties such as fairness, i.e., being non-biased with respect to a given set of attributes, calibration i.e., model confidence being aligned with its predictive accuracy, and explainability, i.e., ability to be understandable to humans. While there has been work focusing on each of these aspects individually, researchers have shied away from simultaneously addressing more than one of these dimensions. In this work, we address the problem of building models which are both fair and calibrated. We work with a specific definition of fairness, which closely matches [Biswas et. al. 2019], and has the nice property that Bayes optimal classifier has the maximum possible fairness under our definition. We show that an existing negative result towards achieving a fair and calibrated model [Kleinberg et. al. 2017] does not hold for our definition of fairness. Further, we show that ensuring group-wise calibration with respect to the sensitive attributes automatically results in a fair model under our definition. Using this result, we provide a first cut approach for achieving fair and calibrated models, via a simple post-processing technique based on temperature scaling. We then propose modifications of existing calibration losses to perform group-wise calibration, as a way of achieving fair and calibrated models in a variety of settings. Finally, we perform extensive experimentation of these techniques on a diverse benchmark of datasets, and present insights on the pareto-optimality of the resulting solutions.

\end{abstract}

%% file: introduction.tex
\section{Introduction}\label{sec:introduction}
Neural models have been shown to provide impressive performance for a large class of applications, including those in computer vision, natural language processing (NLP), speech and reinforcement learning~\cite{goodfellow&al16,sutton&barto18}. It has been argued that for an end-to-end deployment in a real world setting, a machine learning model should have some desriable properties such as, interpretability, i.e., being understandable in their predictions, being fair, i.e., not having any bias with respect to the values of a given (protected) attribute value, and being calibrated, i.e., not making predictions which are over (under) confident. It has also been shown that by themselves, neural models, while being highly accurate, often lack these properties. As a result, several researchers have focused on building models which are interpretable~\cite{murdoch&al19}, fair~\cite{bias_fairness_survey} and calibrated~\cite{calibration_nn_survey}. But to the best of our knowledge, there is very limited work on addressing more than one of these properties simultaneously while still being accurate. Motivated by this observation, our focus in this work is on designing neural models which are both fair and calibrated.

Multiple definitions of fairness have been proposed in the literature~\cite{bias_fairness_survey}. For our current exposition, we work with a variation of the definition proposed by ~\cite{biswas}. We specifically choose this definition since it allows for building models which are both accurate and fair, as long as the predicted (aggregate) probability of the target variable does not deviate (significantly) from that observed in the data, conditioned on the sensitive attribute. Satisfaction of this definition implies that the model is not amplifying the unfairness already present in the real-world data. For calibration, we use the standard definition as existing in the literature~\cite{multiclass_calib}. 

As the main theoretical results of our paper, we show that a negative result shown by ~\cite{kleinberg} regarding joint optimization of fairness and calibration, does not hold in our setting, since our definition of fairness depends on the conditional data distribution, as opposed to others such as equalized odds~\cite{bias_fairness_survey}, which strive for an "absolute" notion of fairness independent of what is observed in the data. We show that calibration a model for each value of the sensitive attribute in fact implies our definition of fairness. Following this, we provide post-processing and train time techniques to trade-off between fairness, calibration and accuracy. 

As a post processing technique, we propose a variant of Temperature scaling~\cite{calibration_nn_survey} which we show can also achieve fairness under our definition. Calibration achieved by temperature scaling on the learned model does not hurt accuracy. Next, we start with the calibration losses proposed in literature~\cite{label_smoothing, focal_loss, dca, mdca, mmce}, and formulate their extensions so that model can be calibrated given each value of the protected attribute. As a {\em hybrid} approach, we first train a model trained using the fairness-calibration loss. We then apply our post-processing technique to improve calibration and fairness without disturbing the accuracy of this learned model.

Finally, we present a detailed analysis of all these techniques on a diverse benchmark of datasets. We also present some insights into finding pareto-optimal points and quantifying tradeoff between fairness, accuracy and calibration.

In Section \ref{sec:related_work}, we provide an overview of prior work in fairness and calibration. Following that, in section \ref{sec:theory} we discuss the definitions of calibration and unfairness used in this paper. We also provide justification for the use of these particular definitions and prove some goodness properties. Next, in section \ref{sec:joint_optimization_techniques} we talk about some techniques to jointly optimise for performance, fairness and calibration. Finally, in section \ref{sec:emperical_evaluations} we provide experimental results of these techniques on a wide range of real world datasets.

%% file: related_work.tex
\section{Related Work}\label{sec:related_work}
\textbf{Definitions of Unfairness}: Over the years, many definitions of measuring unfairness in ML have been proposed. \cite{bias_fairness_survey} summarises all the definitions of unfairness found in literature. Furthermore, \cite{kleinberg} proves that more than two of these definitions cannot be simultaneously satisfied except in some fixed cases. None of these works specifically talk about \textit{amplification} in the bias already present in the data. \cite{biswas} defines two properties that a good definition must satisfy. They require fairness concept to be agonistic of prior probability shifts within groups and require it to hold true for a perfect classifier. They define \textit{Proportional Equality} definition which talks about amplification in the bias. \cite{zhao_men_shopping} also talks about Bias Amplification and provides a way to measure it in multi-class domain.\\
\textbf{Unfairness}: \cite{fairness_regularization} talks about the reasons of unfairness in ML models. It mentions underfitting as one of the reasons and explains that underfitted models make more predictions on prior distribution of the data. It lists other reasons for unfairness and goes on to describe a regularization approach to mitigate it. \cite{underestimation_bias_and_underfitting} further reinforces that underfitting causes unfairness in ML models and provides empirical results which depict the same. \cite{simplicity_bias} talks about extreme Simplicity Bias in Neural Networks, where the neural networks make predictions on simple features even if they are less predictive. \cite{feature_wise_bias_amp} talks about feature selection techniques to introduce less unfairness into neural models.\\
\textbf{Calibration methods for Neural Networks}: \cite{calibration_nn_survey} summarises different methods being used for calibration of neural networks. It talks about method of temperature scaling which is done at test time as being the state-of-the-art as it does not change the labels and hence affect the accuracy. \cite{multiclass_calib} extends this definition to multiclass setting and provides ways to calibrate multiclass models. Calibration at training time is a little tricky as binning schemes introduce non-differentiability and hence adding a loss term is difficult. \cite{label_smoothing, focal_loss, dca, mdca, mmce} design surrogate loss terms which can be optimized at train time to obtain better calibrated models.\\
\textbf{Unfairness using Calibration}: \cite{kleinberg} provides class-wise calibration as a definition of fairness. \cite{on_fairness_and_calib} proves that calibration is not compatible with equalized odds and goes on to relax equalized odds to make it compatible with calibration.\cite{local_calib} provides a kernel-based method to define calibration in a particular region locally. It also provides a post processing technique which they show improves group-wise calibration of the model.\\
\textbf{Other related works}: \cite{fair_rep_disentangle} provides a neural framework to de-correlate every feature from the sensitive feature so that any model trained on the de-correlated data satisfies demographic parity. \cite{why_clf_disc} argues that unfairness introduced by inadequate samples sizes or unmeasured predictive variables should be addressed through data collection. \cite{explainability_fair_ml} introduces a Shapley value paradigm to attribute the model's unfairness to individual input features. \cite{fairness_constraints_ssl} talks about how unlabeled data can be effectively used to obtain better fairness-accuracy trade-off.\\

%% file: theory.tex
\section{Theoretical Framework}\label{sec:theory}
\subsection{Background and Definitions}
We address the problem of supervised multi-class classification with two sensitive groups in this paper. The feature vector $\bX \in \mc{X}$, the label $Y \in \{1, \dots, K\}$ and the sensitive group $A \in \{0, 1\}$ are random variables where $(\bX, Y, A)$ follows the joint distribution $\mc{D}$. Our training and test datasets are sampled i.i.d from $\mc{D}$. We train a model $h : \mc{X} \rightarrow \R^K$ on the training dataset where ${\sf softmax}(h(\bX)) = \hat{\bP}$ where $\hat{P}_k$ is a random variable representing the probability of predicting class $k$. The prediction $\hat{Y}$ given $\bX$ is a random variable distributed as ${\sf Multinoulli}({\sf softmax}(h(\bX)))$.\\
\textbf{Calibration} : A model is said the be calibrated if $\hat{\bP}$ represents the true probability distribution over labels. It is defined as follows in \cite{multiclass_calib}.
\begin{definition}[Perfectly Calibrated Models]\label{def:calib}
    Model $h$ is perfectly calibrated on $(\bX, Y, A) \sim \mc{D}$ if $\forall \, k \in \{1, \dots, K\}$
    \begin{align}
    \label{eq:strong_calib_defn}
        \Pr[Y = k | \hat{\bP} = \bp] = p_k \qquad \forall \, \bp \in \Delta^K
    \end{align}
    Probability is taken over the joint distribution of $(Y, \hat{\bP})$. $\Delta^K$ represents the $K$-simplex.
\end{definition}
This is a strong notion of calibration. It implies the following weaker notion of calibration \cite{multiclass_calib}. Both definitions are equivalent for the binary classification problem.
\begin{definition}[Weakly calibrated models]
    Model $h$ is weakly calibrated on $(\bX, Y, A) \sim \mc{D}$ if $\forall \, p \in [0, 1]$
    \begin{align}
    \label{eq:weak_calib_defn}
        \Pr[Y = {\sf argmax}_k\hat{P}_k | \max_k\hat{P}_k = p] = p
    \end{align}
    Probability is taken over the joint distribution of $(Y, \hat{\bP})$. 
\end{definition}
In order to measure miscalibration, the notion of expected calibration error is defined \cite{calibration_nn_survey}.
\begin{definition}[Expected Calibration Error (ECE)]
    The expected calibration error of $h$ on $(\bX, Y, A) \sim \mc{D}$ is
    \begin{align}
    \label{eq:ece_defn}
        \E_{\hat{\bP}}\left[\left|\Pr[Y = {\sf argmax}_k\hat{P}_k | \max_k\hat{P}_k = p] - p\right|\right]
    \end{align}
\end{definition}
Since we only have finite sample access, and $\hat{P}$ is a continuous random variable, the expectation in \eqref{eq:ece_defn} cannot be computed. Hence, we approximate it by partitioning the unit interval into $M$ equi-width bins. Define $B_m$ as the set of indices of samples whose confidence score falls in $I_m = \left(\tfrac{m-1}{M}, \tfrac{m}{M}\right]$. Given a dataset $\{\bx^{(i)}, y^{(i)}, a^{(i)}\}_{i=1}^n \sim \mc{D}^n$ and a model $h$ such that ${\sf softmax}(h(\bx^{(i)})) = \hat{\bp}^{(i)}$, define $\hat{y}^{(i)} = {\sf argmax}_k\hat{p}_k^{(i)}$ and $\hat{p}^{(i)} = \max_{k}\hat{p}^{(i)}$. We define the average accuracy of $B_m$ as
\begin{align}
    {\sf acc}(B_m) = \tfrac{1}{|B_m|}\sum_{i \in B_m}\mathds{1}\{y^{(i)} = \hat{y}^{(i)}\}
\end{align}
We also define the average confidence of $B_m$ as
\begin{align}
    {\sf conf}(B_m) = \tfrac{1}{|B_m|}\sum_{i \in B_m}\hat{p}^{(i)}
\end{align}
Finally, the empirical estimator of ECE is defined as
\begin{align}
    {\sf ECE} = \sum_{m=1}^M\tfrac{|B_m|}{n}\left|{\sf acc}(B_m) - {\sf conf}(B_m)\right|
\end{align}
\textbf{Fairness}: There are a lot of competing definitions of fairness \cite{bias_fairness_survey}. In this paper we focus on the amplification of the unfairness already present in $\mc{D}$ which is introduced by the model. The unfairness already present in $\mc{D}$ can be measured by $\Pr[Y = k | A = 1]$ and $\Pr[Y = k | A = 0]$ for all $k \in \{1, \dots, K\}$. For example, for the classification task in Table \ref{tab:example}, the distribution is unfairly favours males over females while predicting doctors. This is unfairness already present in the distribution. Model 1 in Table \ref{tab:example} predicts the same fraction of males and females as doctors as the distribution. Thus, we say it does not \text{amplify} unfairness present in the distribution. We define the following notion of fairness.
\begin{definition}[Perfectly Fair Models]\label{def:fairness}
    Model $h$ is perfectly fair on $(\bX, Y, A) \sim \mc{D}$ if $\forall \, k \in \{1, \dots, K\}$
    \begin{align}
        \Pr[\hat{Y} = k | A = 1] = \Pr[Y = k | A = 1] \quad \text{and} \quad
        \Pr[\hat{Y} = k | A = 0] = \Pr[Y = k | A = 0]
    \end{align}
\end{definition}
The following lemma follows since any model $h$ such that $\hat{Y}\, |\, X, A \,\,{\buildrel d \over =}\,\,
Y \,|\, X, A$ satisfies the definition \ref{def:fairness}. This gives us a very desirable property for the definition of fairness.
\begin{lemma}
    A perfect classifier is a perfectly fair model.
\end{lemma}
To measure unfairness of a model, we use a variation of the \textit{Proportional Equality} definition proposed by \cite{biswas}. In probabilistic terms, the definition is as follows.
\begin{definition}[Proportional Equality]
    The proportional equality of a model $h$ on $(\bX, Y, A) \sim \mc{D}$ is
    \begin{align}
    \label{eq:pe}
        {\sf PE} = \underset{k \in \{1, \dots, K\}}{\max}\left\{\left|\tfrac{\Pr[Y = k| A = 1]}{\Pr[Y = k| A = 0]} - \tfrac{\Pr[\hat{Y} = k| A = 1]}{\Pr[\hat{Y} = k| A = 0]}\right|\right\}
    \end{align}
\end{definition}
For empirical evaluation, we approximate the first term in \eqref{eq:pe} using the true labels and group labels in the training dataset. We approximate the second term using the model predictions and group labels on the test dataset. In the example in Table \ref{tab:example}, Model 1 is perfectly fair and it's ${\sf PE} = 0$. On the other hand, Model 2 has a ${\sf PE}$-unfairness of $\max\{(0.8/0.7) - (0.15/0.2), (0.85/0.8) - (0.2/0.3)\} = 0.396$. Intuitively, {\sf PE}-unfairness penalises the amplification in the unfairness that the model introduces in proportion to the under-representation of the minority group.\\
To approximate the second term, if we use the soft confidence score of the $k^{th}$ class $(\hat{p}_k)$ then we call it \textit{stochastic proportional equality} and if we use the indicator of the prediction $(\mathds{1}\{{\sf argmax}_kp_k = k\})$ then we call it \textit{deterministic proportional equality}.\\
\textbf{Group-wise calibration}: \cite{kleinberg} introduces group-wise calibration as a way to define fairness. It is defined as follows
\begin{definition}[Group-wise Calibration]
    A model $h$ is group-wise calibrated on $(\bX, Y, A) \sim \mc{D}$ if $\forall \, a \in \{0, 1\}$ and $\forall \, k \in \{1, \dots, K\}$,
    \begin{align}\label{eqn:group_wise_calib_defn}
        \Pr[ Y = k | \hat{\bP} = \bp, A = a] = p_k \qquad \forall\, \bp \in \Delta^K
    \end{align}
    Probability is taken over joint distribution of $(A, Y, \hat{\bP})$. $\Delta^K$ represents the $K$-simplex.
\end{definition}
\subsection{Key Results}
In this section, we prove the key results of this paper. The following lemma follows trivially using the total probability theorem.
\begin{lemma}\label{lem:group_wise_calib_calib}
    A group-wise calibrated model is perfectly calibrated.
\end{lemma}

\citet{kleinberg} show that group-wise calibration and statistical parity cannot be simultaneously achieved unless base rates match (i.e. $\Pr[A = 0] = \Pr[A = 1]$). They also show that group-wise calibration and equalized-odds cannot be achieved simultaneously unless either the base rates match for the sensitive groups or the model predicts perfectly (i.e. it knows with certainty whether a particular example belongs to a particular class).\\
We show that group-wise calibration and Proportional Equality can be simultaneously achieved. We in fact show that group-wise calibration is a stronger condition than {\sf PE}-fairness in the following lemma.
\begin{lemma}\label{lem:group_wise_calib_fair}
    A group-wise calibrated model is perfectly fair.
\end{lemma}
\begin{proof}
    Let $h$ be group-wise calibrated and $a \in \{0, 1\}$. By definition,
    \begin{align}\label{eqn:pred_cond_prob}
        \Pr[\hat{Y} = k | A = a] = \E_{\hat{\bP}}\left[p_k | A = a\right]
    \end{align}
    Since $h$ is group-wise calibrated, taking conditional expectation with respect to $\hat{\bP}$ when $A = a$ in \eqref{eqn:group_wise_calib_defn},
    \begin{align}\label{eqn:group_wise_calib_expec}
        \E_{\hat{\bP}}\left[\Pr[Y = k | \hat{\bP} = \bp, A = a] | A = a\right] = \E_{\hat{\bP}}[p_k | A = a]
    \end{align}
    The LHS in \eqref{eqn:group_wise_calib_expec} is equal to $\Pr[Y = k | A = a]$ and the RHS in \eqref{eqn:group_wise_calib_expec} is equal to $\Pr[\hat{Y} = k | A = a]$ by \eqref{eqn:pred_cond_prob}. This completes the proof.
\end{proof}
Lemma \ref{lem:group_wise_calib_calib} and \ref{lem:group_wise_calib_fair} prove that achieving group-wise calibration can lead to joint optimization of fairness and calibration. In the next section, we try to build upon this insight to propose methods to achieve group-wise calibration whilst jointly optimising for accuracy.

%% file: joint_optimization_techniques.tex
\section{Joint Optimization Techniques}\label{sec:joint_optimization_techniques}

\subsection{Post-processing techniques}
\textbf{Dual Temperature Scaling} : Temperature scaling has been known to be a very effective method to calibrate neural networks ~\cite{calibration_nn_survey}. A \textit{single} temperature parameter $T > 0$ is tuned on the validation dataset to improve the calibration of confidence scores. Let $h(\bx) = \bz$, then the new confidence prediction is given by,
\begin{align}
    \hat{q}_k = {\sf softmax}(z_i/T) \qquad \forall \, k \in \{1, \dots, K\}
\end{align}
This temperature $T$ is optimized with respect to the cross entropy loss on the validation set. At $T = 1$, we recover the original predictions $\hat{p}_k$. As $T \rightarrow \infty$, $\hat{q}_k \rightarrow 1/K$ and as $T \rightarrow 0$, $\hat{\mathbf{q}}$ becomes a point mass at ${\sf argmax}_kz_k$. \\
We modify this method by fine-tuning a separate temperature parameter for each sensitive group. At inference time, we choose the temperature parameter based on the sensitive group of the example on which inference is being drawn.\\
One major advantage of this technique is that model prediction $({\sf argmax}_kq_k)$ does not change. Hence, the accuracy of the model remains unchanged after temperature scaling. This makes it an ideal post-processing algorithm.\\

\subsection{Train-time techniques}\label{sec:train_time_techniques}
Train time losses have been proposed in the literature to achieve better calibrated neural models. We reweigh these loss terms to give equal weights to minimization of these losses for both the sensitive groups. We classify these loss terms in two major parts.
\subsubsection{Linear loss functions}
These are loss functions of the form $l: \{1, \dots, K\} \times \R^K \rightarrow \R$. Thus, loss of each example can be computed independently from other examples. The loss over a batch $B$ is defined as ${\sf L}(B) := \tfrac{1}{|B|}\sum_{i \in B}l(y^{(i)}, h(\bx^{(i)}))$. We partition our training batch $B$ into sub-batches for each sensitive group ($B_a$ for $a \in \{0, 1\}$). We individually compute the loss on each of these sub-batches and then take a convex combination. Intuitively, we assign equal importance to minimization of this loss for both the sensitive groups and hence it should lead to group-wise calibration. We take a convex combination instead of adding the loss terms as we observe it gives better empirical results. We tune the parameter of the convex combination as a hyperparameter to the experiment. We define the group-wise loss as follows.
\begin{align}\label{eq:grp_loss_lin}
    {\sf L_g}(B) := (1-\rho){\sf L}(B_0) + \rho{\sf L}(B_1) \quad \rho \in [0, 1]
\end{align}
Here $\rho$ is the convex combination parameter. Taking $\rho=1/2$ assigns equal weight to loss on each group. Taking $\rho = \Pr[A = 1]$, the group-wise loss collapses into the loss insensitive to group labels. Hence, we perform a grid search between the two values.\\
Label smoothing ({\sf LS}) ~\cite{label_smoothing}, Focal loss({\sf FL})~\cite{focal_loss} and sample dependent focal loss({\sf FLSD})~\cite{focal_loss} are loss functions which are used in lieu of cross entropy loss to train better calibrated models. For these losses, we train directly using the group-wise loss. Difference between calibration and accuracy ({\sf DCA})~\cite{dca} and it's multidimensional variant({\sf MDCA})~\cite{mdca} are loss functions used in addition to the cross entropy loss to improve calibration. For these losses, we train with the following loss.
\begin{align}\label{eq:grp_loss_w_nll}
    {\sf NLL}(B) + \lambda {\sf L_g}(B)
\end{align}
The hyperparameter $\lambda$ determines the trade-off between optimizing for accuracy and optimizing for calibration.
\subsubsection{Pair-wise loss functions}
These are loss functions of the form $l: (\{1, \dots, K\} \times \R^K) \times (\{1, \dots, K\} \times \R^K) \rightarrow \R$. Loss is defined for a pair of examples. Thus, given a batch $B$ we define the loss as ${\sf L}(B) := \tfrac{1}{|B|^2}\sum_{i, j \in B}l(y^{(i)}, h(\bx^{(i)}), y^{(j)}, h(\bx^{(j)}))$. We extend this definition naturally to define our loss. Given two batches $B$ and $B'$,
\begin{align}
    {\sf L}(B, B') := \tfrac{1}{|B||B'|}\sum_{i \in B}\sum_{j \in B'}l(y^{(i)}, h(\bx^{(i)}), y^{(j)}, h(\bx^{(j)}))
\end{align}
We partition a batch $B$ based on it's sensitive group into $B_0$ and $B_1$. We define our group-wise loss term as follows.
\begin{align}\label{eq:grp_loss_pairwise}
    {\sf L_g}(B) := (1-\rho)^2{\sf L}(B_0, B_0) + \rho^2{\sf L}(B_1, B_1) + 2\rho(1-\rho){\sf L}(B_0, B_1) \qquad \rho \in [0, 1]
\end{align}
The interpretation of $\rho$ is the same as that for \eqref{eq:grp_loss_lin}. This can be seen by that fact that ${\sf L}(B, B') = {\sf L}(B', B)$.\\
Maximum mean calibration error ({\sf MMCE}) and it's weighted version ({\sf MMCE-W})~\cite{mmce} are such loss functions. They are trained along with the cross entropy term. Hence we use the loss term in \eqref{eq:grp_loss_w_nll} to train models using these losses.

%% file: empirical_evaluation.tex
\section{Empirical Evaluation}\label{sec:emperical_evaluations}
\subsection{Datasets}
We use a set of binary classification datasets to evaluate these techniques. The details of these datasets are as listed below.
\begin{itemize}
    \item [1.] \textbf{Adult}: UCI dataset where each entry represents of a person and the objective is to predict whether they earn ${\sf >50K}$ or ${\sf <=50K}$ annually. We use the ${\sf sex}$ of the person for the sensitive group.
    \item [2.] \textbf{Arrhythmia}: UCI dataset where given an example the target is to predict presence or absence of cardiac arrhythmia. The sensitive group is given by gender.
    \item [3.] \textbf{Communities and Crime}: UCI dataset where each example represents a community and the task is to predict whether the community has a violent crime rate in $70^{th}$ percentile of all communities. The sensitive group is given by whether the community is has a majority of white population.
    \item [4.] \textbf{Drug}: UCI dataset where the task is to classify weather a person is a drug consumer or not. The sensitive group is given by race.
    \item [5.] \textbf{Compas}: Criminal recidivism dataset where the task is to predict recidivism of a person based on history. The sensitive group is given by race.
    \item [6.] \textbf{German}: UCI dataset where the task is to classify good or bag credit for a person. The sensitive group is given by gender.
    \item [7.] \textbf{Lawschool}: UCI dataset where the target is to predict whether a person passed the bar exam. The sensitive group is given by gender.
\end{itemize}
Table \ref{tab:dataset_stats} contains statistical details of all these datasets. All the features in these datasets are categorical. Thus, they are converted to multi-hot encoding vectors and suitably hashed. We follow the pre-processing steps of \cite{ensuring_fairness_beyond_training_data} or \cite{robust_fairness_covariate_shift}. We divide each dataset into train, validation and test sets with ratio of $6: 1: 1$.

\begin{figure}
\begin{minipage}{0.49\linewidth}
\input{table_example}
\end{minipage}
\hfill\vline\hfill
\begin{minipage}{0.49\linewidth}
\input{table_dataset_stats}
\end{minipage}
\end{figure}

\subsection{Training details}
Using the train-time methods described above, a $2$-layer perceptron is trained using a fixed learning rate of $1e-4$. It has $128$ and $64$ node hidden layers followed by a $2$ node output layer, all with {\sf Relu} activation, for binary classification. The confidence scores for both classes are computed by taking a {\sf softmax} over the output. We train this network for $500$ epochs for every configuration using the {\sf Adam} optimizer.\\
We compute the {\sf ECE} as well as the \textit{deterministic} and \textit{stochastic} {\sf PE} metrics along with the accuracy on the test set at each epoch of training. We also perform \textit{Dual Temperature Scaling} at each epoch of the training. We perform these experiments across $5$ random seeds and report the average metric values. All training is done sequentially on a single GPU.

\subsection{Empirical observations}

\subsubsection{Dual Temperature Scaling}
As dual temperature scaling only involves tuning two parameters, we implement it at each epoch of training. We fine-tune these temperatures using the cross entropy loss on the validation dataset. We use a learning rate of $1e-4$ and {\sf Adam} optimizer. We set the maximum number of epochs for fine-tuning as $500$ but we implement early stopping so that the fine-tuning stops once the {\sf ECE} starts increasing on the validation set.\\
Since dual temperature scaling does not change the prediction of the model, the accuracy and \textit{deterministic} {\sf PE}-fairness do not change. We thus report {\sf ECE} and \textit{stochastic} {\sf PE}-fairness on the test dataset. Table \ref{tab:t_scal} gives the percentage improvement that dual temperature scaling offers for these metrics when applied after training with cross-entropy loss. The percentage improvement of the best value after temperature scaling across epochs over the best value before temperature scaling across epochs is reported.\\
It can be observed that dual temperate scaling offers improvement in {\sf ECE} in all the datasets. It also improves the \textit{stochastic} {\sf PE}-fairness in $4$ of the $7$ datasets. The highest improvement in both these metrics can be observed on the Compas dataset. This is because the features of Compas dataset are less informative. The maximum accuracy achieved upon training on Compas dataset is $67\%$ whereas all the other datasets achieve accuracy of above $80\%$. Thus, it is harder to train on and exhibits more headroom to improve fairness and calibration. On the other hand, fairness worsens on Adult, Arrythmia and Drug datasets. We believe this is because model is well-calibrated and fair when trained with cross-entropy loss.

\subsubsection{Train-time techniques}\label{sec:train_time_tech}
We train our model using the losses described in Section \ref{sec:train_time_techniques}. The details of the hyperparameters $\{\rho, \lambda\}$ over which we search can be found in Appendix \ref{appendix:training_details}. In this section, we first discuss how effectively these losses can optimize purely for calibration and fairness. Then we give some insights about pareto-optimality achieved between fairness and calibration for a loss given a particular accuracy.\\

\noindent\textbf{Optimizing for fairness}: For each dataset and technique, we compute the average of minimum \textit{stochastic} {\sf PE}-fairness obtained over all seeds. We use this average {\sf PE}-fairness to determine which technique performs best on a particular dataset. Table \ref{tab:fair} reports the best technique on each dataset. We also report the average (across seeds) percentage improvement in \textit{stochastic} {\sf PE}-fairness and average percentage change in {\sf ECE} and Accuracy using this technique.\\
We can observe a trade-off in Table \ref{tab:fair}. \textit{Stochastic} {\sf PE}-fairness can be improved substantially by sacrificing some amount of {\sf ECE} and Accuracy. {\sf MMCE} and {\sf MMCE-W} perform best on most datasets. {\sf MDCA} and {\sf DCA} perform the best for German and Lawschool respectively. However they incur more loss of Accuracy and gain in {\sf ECE}. A similar analysis using \textit{deterministic} {\sf PE}-fairness is presented in Appendix \ref{appendix:deterministic_unfairness}.

\begin{figure}
\begin{minipage}{0.49\linewidth}
\input{table_temp_scaling}
\end{minipage}
\hfill\vline\hfill
\begin{minipage}{0.49\linewidth}
\input{table_fairness}
\end{minipage}
\end{figure}

\noindent\textbf{Optimizing for calibration}: We perform similar analysis as the previous section using {\sf ECE} instead of \textit{stochastic} {\sf PE}-fairness. Table \ref{tab:calib} reports the average (over seeds) percentage change in \textit{stochastic} {\sf PE}-fairness, {\sf ECE} and Accuracy when we optimize only for {\sf ECE}.\\
Table \ref{tab:calib} shows that optimizing solely for {\sf ECE} can help increase the accuracy but definitely hurts the fairness of the model. {\sf DCA} perform the best for $5$ datasets. It also improves the accuracy for Adult, Communities and German datasets (at the point of best {\sf ECE}). {\sf MMCE-W} also offers significant improvement in Accuracy for Arrythmia dataset along with improving {\sf ECE}.

\begin{figure}
\begin{minipage}{0.49\linewidth}
\input{table_calib}
\end{minipage}
\hfill\vline\hfill
\begin{minipage}{0.49\linewidth}
\centering
\caption{Pareto-optimal curve: Adult: Accuracy in (83.14\%, 88.14\%)}
\label{fig:adult}
\includegraphics[width=\linewidth]{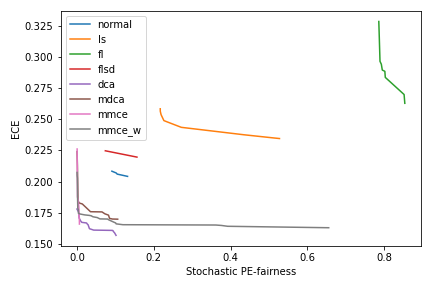}
\end{minipage}
\end{figure}

\begin{figure}
\begin{minipage}{0.49\linewidth}
\centering
\caption{Pareto-optimal curve: Communities: Accuracy in (86.2\%, 91.2\%)}
\label{fig:communities}
\includegraphics[width=\linewidth]{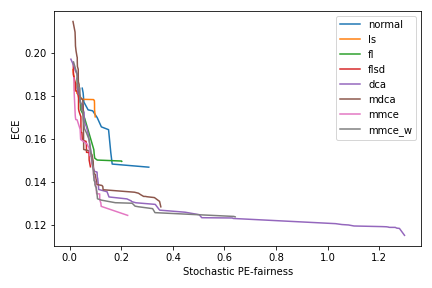}
\end{minipage}
\hfill\vline\hfill
\begin{minipage}{0.49\linewidth}
\centering
\caption{Pareto-optimal curve: Arrhythmia: Accuracy in (89.74\%, 94.74\%)}
\label{fig:arrhythmia}
\includegraphics[width=\linewidth]{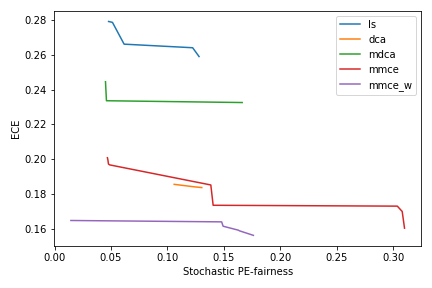}
\end{minipage}
\end{figure}

\begin{figure}
\begin{minipage}{0.49\linewidth}
\centering
\caption{Pareto-optimal curve: Compas: Accuracy in (63.37\%, 68.37\%)}
\label{fig:compas}
\includegraphics[width=\linewidth]{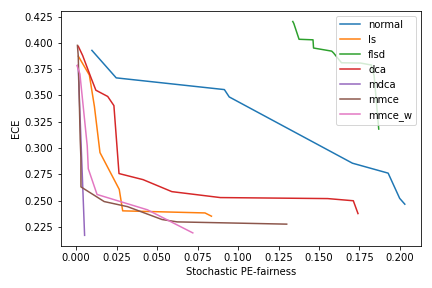}
\end{minipage}
\hfill\vline\hfill
\begin{minipage}{0.49\linewidth}
\centering
\caption{Pareto-optimal curve: Drug: Accuracy in (78.47\%, 83.47\%)}
\label{fig:drug}
\includegraphics[width=\linewidth]{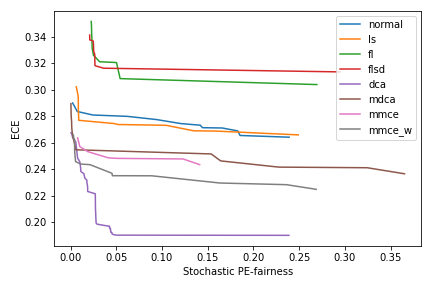}
\end{minipage}
\end{figure}

\begin{figure}
\begin{minipage}{0.49\linewidth}
\centering
\caption{Pareto-optimal curve: German: Accuracy in (81.4\%, 86.4\%)}
\label{fig:german}
\includegraphics[width=\linewidth]{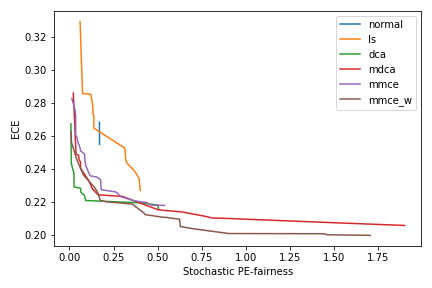}
\end{minipage}
\hfill\vline\hfill
\begin{minipage}{0.49\linewidth}
\centering
\caption{Pareto-optimal curve: Lawschool: Accuracy in (80.09\%, 85.09\%)}
\label{fig:lawschool}
\includegraphics[width=\linewidth]{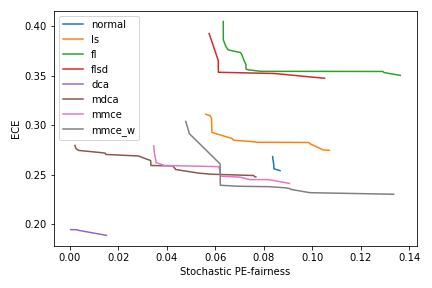}
\end{minipage}
\end{figure}

\noindent\textbf{Pareto-Optimality}: In this section, we try to give some insights of the fairness-calibration trade-off offered by different techniques. To perform this analysis, we first fix the slack in accuracy that we are willing to allow. Given this slack, for each loss we identify (\textit{stochastic} {\sf PE}-fairness, {\sf ECE}) pairs for which the loss in accuracy (from the best accuracy obtained over all train-time techniques) is less than the slack. For every dataset and train-time technique, we do this for all the hyperparameter settings and all seeds. Thus for every dataset and train-time technique, we get a collection of points $\{({\sf PE}_i, {\sf ECE}_i)\}_{i=1}^N$. Given such a collection of points, for each dataset and train-time technique we identify pareto-optimal points. Since lower values of \textit{stochastic} {\sf PE}-fairness and {\sf ECE} indicate fairer and more calibrated models respectively, we define a point $({\sf PE}, {\sf ECE})$ as pareto-optimal if
\begin{align}
    \nexists i \in \{1, \dots, N\} \text{ s.t. } {\sf PE}_i \leq {\sf PE} \text{ and } {\sf ECE}_i \leq {\sf ECE}
\end{align}
We compute this set of pareto-optimal points for each dataset and each train-time technique and plot them by interpolation. Figures \ref{fig:adult}, \ref{fig:arrhythmia}, \ref{fig:communities}, \ref{fig:compas}, \ref{fig:drug}, \ref{fig:german}, \ref{fig:lawschool} show these plots for all the datasets for a 5\% point slack in accuracy. Absence of a particular technique in a plot indicates that it does not ever achieve accuracy loss of less than 5\% points. The curves report absolute number of {\sf ECE} and \textit{stochastic} {\sf PE}-fairness.\\
In each of these plots, curves closer to the origin symbolize better fairness-calibration trade-off. It can be seen that at least one train-time technique always has better trade-off as compared to cross-entropy training. {\sf MMCE} and {\sf MMCE-W} have the best trade-off in most of the datasets. We also observe that {\sf FL} and {\sf FLSD} have a worse trade-off as compared to cross-entropy training on most datasets.

\textbf{Hybrid method}: Since dual temperature scaling is a post-processing technique, it can be applied in conjunction to any of the train-time techniques to further improve \textit{stochastic} {\sf PE}-fairness and {\sf ECE}. We present the same analysis as that in Section \ref{sec:train_time_tech} for this hybrid method in Appendix \ref{appendix:hybrid_method}\\

%% file: table_example.tex
\scriptsize
\centering
\captionof{table}{An example to explain the definition of unfairness. The classification task is to predict is the person is a doctor or a nurse and the sensitive attribute is the gender of the person.}
\label{tab:example}
\begin{tabular}{l|cc}
\toprule
 & $\Pr[Doc|Woman]$ & $\Pr[Doc|Man]$ \\ \midrule
Data distribution & 0.2 & 0.7 \\
Model 1 & 0.2 & 0.7 \\
Model 2 & 0.15 & 0.8 \\ \bottomrule
\end{tabular}%

%% file: table_dataset_stats.tex
\scriptsize
\centering
\captionof{table}{Dataset statistics}
\label{tab:dataset_stats}
\begin{tabular}{lccccc}
\toprule
\textbf{Dataset} & \textbf{Size}& \textbf{d} & $\Pr$[A = 1] & $\Pr$[Y=1|A=0] & $\Pr$[Y=1|A=1] \\ \midrule
Adult & 2020& 97 & 0.74 & 0.25 & 0.59 \\
Arrhythmia & 452& 279 & 0.55 & 0.41 & 0.65 \\
Communities & 1994& 122 & 0.71 & 0.36 & 0.84 \\
Compas & 5278& 11 & 0.6 & 0.61 & 0.49 \\
Drug & 1885& 10 & 0.91 & 0.83 & 0.79 \\
German & 1000& 20 & 0.85 & 0.60 & 0.72 \\
Lawschool & 1823& 17 & 0.54 & 0.51 & 0.55 \\ \bottomrule
\end{tabular}

%% file: table_temp_scaling.tex
\scriptsize
\centering
\captionof{table}{Percentage improvement in \textit{stochastic} {\sf PE}-fairness and {\sf ECE} represented as \%fair and \%calib. respectively.}
\label{tab:t_scal}
\begin{tabular}{lrr}\toprule
\textbf{Dataset}     & \multicolumn{1}{l}{\textbf{\%fair}} & \multicolumn{1}{l}{\textbf{\%calib.}} \\ \midrule
Adult       & -0.89                                 & 0.88                                      \\
Arrhythmia  & -0.10                                  & 0.25                                      \\
Communities & 0.21                                  & 0.80                                       \\
Compas      & 0.57                                  & 2.29                                      \\
Drug        & -4.84                                 & 0.49                                      \\
German      & 0.41                                  & 0.33                                      \\
Lawschool   & 0.07                                  & 0.61                                      \\ \bottomrule
\end{tabular}

%% file: table_fairness.tex
\scriptsize
\centering
\captionof{table}{Percentage change in \textit{stochastic} {\sf PE}-fairness, {\sf ECE} and Accuracy represented as \%fair, \%calib. and \%acc. respectively when optimizing for fairness.}
\label{tab:fair}
\begin{tabular}{llrrr}\toprule
\textbf{Dataset} & \textbf{Best technique} & \textbf{\%fairness} & \textbf{\%calib.} & \textbf{\%acc.} \\ \midrule
Adult            & MMCE               & 97.94               & -14.96            & -3.69          \\
Arrhythmia       & MMCE-W             & 86.10                & -29.55            & -9.36          \\
Communities      & MMCE               & 87.66               & -25.68            & -2.54          \\
Compas           & MMCE               & 92.53               & -23.61            & -2.67          \\
Drug             & MMCE-W             & 82.12               & -4.66             & -7.81          \\
German           & MDCA               & 96.95               & -40.66            & -10.93         \\
Lawschool        & DCA                & 97.48               & -107.48           & -46.35         \\ \bottomrule
\end{tabular}

%% file: table_calib.tex
\scriptsize
\centering
\captionof{table}{Percentage change in \textit{stochastic} {\sf PE}-fairness, {\sf ECE} and Accuracy represented as \%fair, \%calib. and \%acc. respectively when optimizing for calibration.}
\label{tab:calib}
\begin{tabular}{llrrr}\toprule
\textbf{Dataset} & \textbf{Best technique} & \textbf{\%fair} & \textbf{\%calib.} & \textbf{\%acc.} \\ \midrule
Adult            & DCA                & -591.51         & 16.67             & 2.37            \\
Arrhythmia       & MMCE-W             & -578.42         & 17.60              & 3.87            \\
Communities      & DCA                & -2953.28        & 15.91             & 0.11            \\
Compas           & MMCE-W             & -870.73         & 25.60              & -1.93           \\
Drug             & DCA                & -7121.60         & 18.90              & -1.08           \\
German           & DCA                & -341.72         & 12.25             & 0.08            \\
Lawschool        & DCA                & -59.53          & 21.25             & -0.94           \\ \bottomrule
\end{tabular}

%% file: conclusion.tex
\section{Conclusion}
Our work shows that joint optimization of fairness and calibration is achievable when fairness is defined for ensuring that models do not amplify existing biases in the data. We propose post processing techniques and train time techniques for this joint optimization. We give a detailed analysis on performance of these techniques on a diverse benchmark of datasets. We believe that our work addresses significant gap in study of fairness and calibration together. In future work, other novel techniques to achieve group-wise calibration can be developed to aid this joint optimization.

%% file: appendix.tex
\section{Appendix}
\subsection{Further details of experiments}\label{appendix:training_details}
In this section we discuss the losses used by our train-time techniques. We also give details of the hyperparameters used to train each of these models. The following are the details of the losses.
\begin{itemize}
    \item {\sf LS}: Label smoothing was proposed in \cite{label_smoothing} as a loss to train better calibrated models. Instead of computing the cross-entropy loss using the one-hot ground truth vector $\mathbf{q}$, we use a smoothed vector $\bs$ such that $s_i = (1 - \alpha)q_i + \alpha(1 - q_i)/(K-1)$. In our experiments, we use $\alpha = 0.05$.
    \item {\sf FL}: Focal loss mentioned in \cite{focal_loss} is used to train well-calibrated models. It is defined as follows
    \begin{align}
        {\sf FL}(y, \bp) = -(1 - p_y)^\gamma\log(p_y)
    \end{align}
    We use $\gamma = 3$ for our experiments.
    \item {\sf FLSD}: Sample dependent focal loss is also introduced in \cite{focal_loss}. It uses $\gamma = 5$ when $p_y \in [0, 0.2]$ and $\gamma = 3$ otherwise. We use the same for our experiments.
    \item {\sf DCA}: This is a loss used in conjunction with cross-entropy to improve calibration of trained models. On a set of labels and predictions $\{y^{(i)}, \bp^{(i)}\}_{i=1}^m$, it is defined as follows.
    \begin{align}
        \left|\tfrac{1}{m}\sum_{i=1}^m\mathds{1}\{y^{(i)} = {\sf argmax}_k p_k^{(i)}\} - \tfrac{1}{m}\sum_{i=1}^mp_{{\sf argmax}_k p_k^{(i)}}\right|
    \end{align}
    \item {\sf MDCA}: A variant of {\sf DCA} which performs better in multiclass settings proposed in \cite{mdca}. It is also trained after taking a linear combination with the cross-entropy loss. On a set of labels and predictions $\{y^{(i)}, \bp^{(i)}\}_{i=1}^m$, it is defined as follows.
    \begin{align}
        \frac{1}{K}\sum_{k=1}^K\left|\tfrac{1}{m}\sum_{i=1}^mp_k^{(i)} - \tfrac{1}{m}\sum_{i=1}^m\mathds{1}\{y^{(i)} = k\}\right|
    \end{align}
    \item {\sf MMCE}: A kernel-based pair-wise loss for calibration presented in \cite{mmce}. It is used in linear combination with cross-entropy loss.On a set of labels and predictions $\{y^{(i)}, \bp^{(i)}\}_{i=1}^m$, if we define $c^{(i)} := \mathds{1}\{y^{(i)} = {\sf argmax}_k p_k^{(i)}\}$ and $r^{(i)} := p_{{\sf argmax}_k p_k^{(i)}}$, then the square of the loss is defined as
    \begin{align}
        \sum_{i, j  = 1}^m\frac{(c^{(i)} - r^{(i)})(c^{(j)} - r^{(j)})k(r^{(i)}, r^{(j)})}{m^2}
    \end{align}
    We use a laplacian kernel with $\gamma=0.2$.
    \item {\sf MMCE-W}: This is an extension of the {\sf MMCE} loss proposed in \cite{mmce}. Is is also used in linear combination with cross-entropy loss. Let there be $m_1$ elements with $c^{(i)} = 1$ and $m_0$ elements with $c^{(i)} = 0$. The square of the loss is defined as
    \begin{align}
        \sum_{c^{(i)} = c^{(j)} = 0}\frac{r^{(i)}r^{(j)}k(r^{(i)}, r^{(j)})}{m_0^2} + \sum_{c^{(i)} = c^{(j)} = 1}\frac{(1 -r^{(i)})(1 -r^{(j)})k(r^{(i)}, r^{(j)})}{m_1^2} - 2\sum_{c^{(i)} = 1, c^{(j)} = 0}\frac{(1 -r^{(i)})r^{(j)}k(r^{(i)}, r^{(j)})}{m_0m_1}
    \end{align}
\end{itemize}
Further, for {\sf DCA}, {\sf MDCA}, {\sf MMCE} and {\sf MMCE-W}, we perform a grid search over $\lambda$ by taking its values in $\{0.2, 0.5, 1.0, 2.0, 3.0, 4.0, 5.0, 10.0, 20.0, 30.0, 40.0, 50.0\}$.\\
We choose the values of $\rho$ to perform a grid search based on the dataset. They are listed below.
\begin{itemize}
    \item Adult: $\{0.4, 0.45, 0.5, 0.55, 0.6, 0.65, 0.70, 0.75, 0.8\}$
    \item Arrhythmia: $\{0.4, 0.45, 0.5, 0.55, 0.6\}$
    \item Communities: $\{0.4, 0.45, 0.5, 0.55, 0.6, 0.65, 0.70, 0.75\}$
    \item Compas: $\{0.4, 0.45, 0.5, 0.55, 0.6, 0.65\}$
    \item Drug: $\{0.4, 0.45, 0.5, 0.55, 0.6, 0.65, 0.70,$
    \\$ 0.75, 0.8, 0.85, 0.9, 0.95\}$
    \item German: $\{0.4, 0.45, 0.5, 0.55, 0.6, 0.65,$\\
    $ 0.70, 0.75, 0.8, 0.85, 0.9\}$
    \item Lawschool: $\{0.4, 0.45, 0.5, 0.55, 0.6\}$
\end{itemize}
\subsection{Analysis for deterministic unfairness}\label{appendix:deterministic_unfairness}
We replicate the table \ref{tab:fair} with \textit{deterministic} {\sf PE}-fairness instead of \textit{stochastic} {\sf PE}-fairness in table \ref{tab:hard_fair}.

\subsection{Hybrid method analysis}\label{appendix:hybrid_method}
In this section we perform the same analysis as Section \ref{sec:train_time_tech} for \textit{hybrid} methods. Hybrid method involve applying \textit{dual temperature scaling} after any train-time technique. We replicate Table \ref{tab:fair} using the hybrid method in table \ref{tab:fair_hybrid}. We can notice that the best techniques do not change from that in Table \ref{tab:fair}. There is an improvement in the percentage gain in fairness in some cases.\\
Similarly, we replicate the Table \ref{tab:calib} using hybrid methods in table \ref{tab:calib_hybrid}. Again, the best techniques do not change from that in Table \ref{tab:calib}. We add the hybrid methods to Figures \ref{fig:adult}, \ref{fig:arrhythmia}, \ref{fig:communities}, \ref{fig:compas}, \ref{fig:drug}, \ref{fig:german}, \ref{fig:lawschool} and construct Figures \ref{fig:adult_hybrid}, \ref{fig:arrhythmia_hybrid}, \ref{fig:communities_hybrid}, \ref{fig:compas_hybrid}, \ref{fig:drug_hybrid}, \ref{fig:german_hybrid}, \ref{fig:lawschool_hybrid}. We can see that after dual temperature scaling, the pareto-optimal curve of the \textit{hybrid} technique closely follows the curve without temperature scaling. In most cases, it offers a better trade-off between fairness and accuracy as compared to the curve without temperature scaling. The \textit{hybrid} methods are indicated by '\_ts' suffix at the end of each technique.

\FloatBarrier

\begin{figure}
\begin{minipage}{0.49\linewidth}
\input{table_hard_fairness}
\end{minipage}
\hfill\vline\hfill
\begin{minipage}{0.49\linewidth}
\input{table_fair_hybrid}
\end{minipage}
\end{figure}

\begin{figure}
\begin{minipage}{0.49\linewidth}
\input{table_calib_hybrid}
\end{minipage}
\hfill\vline\hfill
\begin{minipage}{0.49\linewidth}
\centering
\caption{Pareto-optimal curve: Adult: Accuracy in (83.14\%, 88.14\%) with hybrid techniques}
\includegraphics[width=\linewidth]{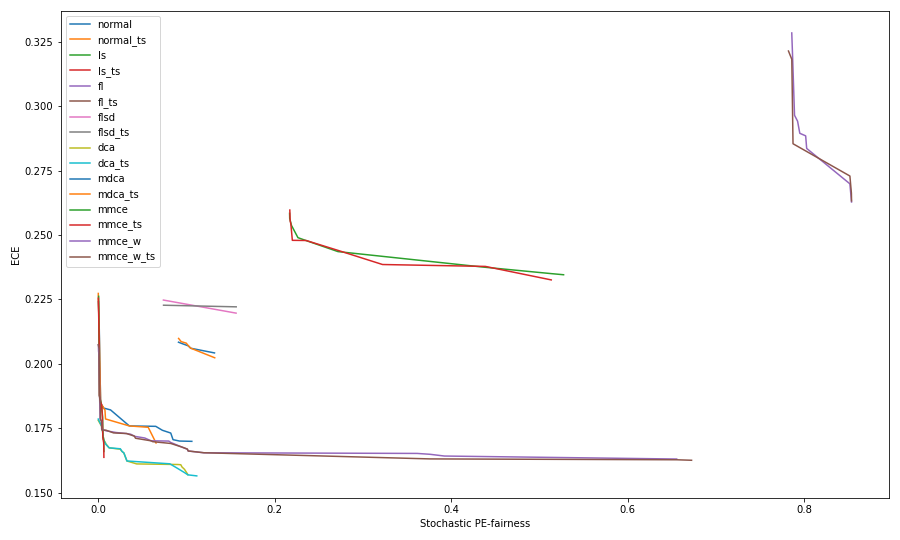}
\label{fig:adult_hybrid}
\end{minipage}
\end{figure}

\begin{figure}
\begin{minipage}{0.49\linewidth}
\centering
\caption{Pareto-optimal curve: Communities: Accuracy in (86.2\%, 91.2\%) with hybrid techniques}
\includegraphics[width=\linewidth]{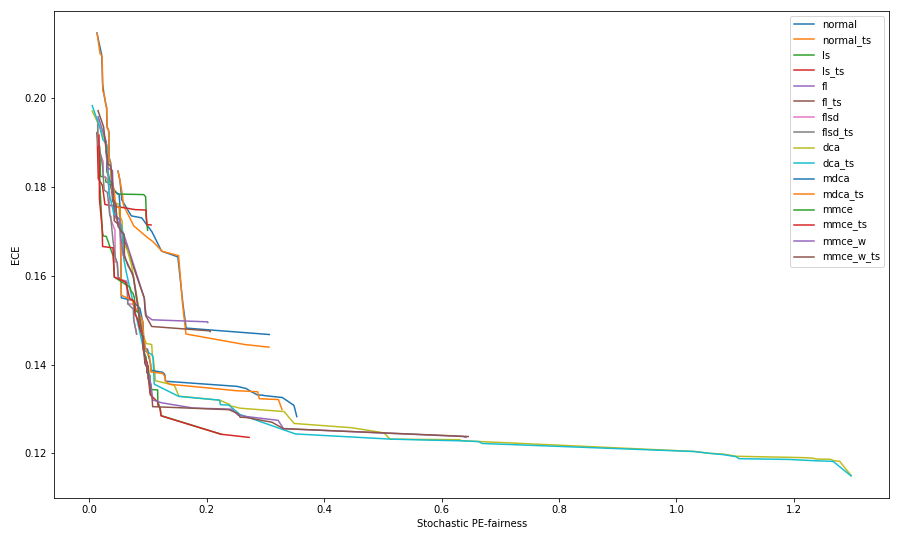}
\label{fig:communities_hybrid}
\end{minipage}
\hfill\vline\hfill
\begin{minipage}{0.49\linewidth}
\centering
\caption{Pareto-optimal curve: Arrhythmia: Accuracy in (89.74\%, 94.74\%) with hybrid techniques}
\includegraphics[width=\linewidth]{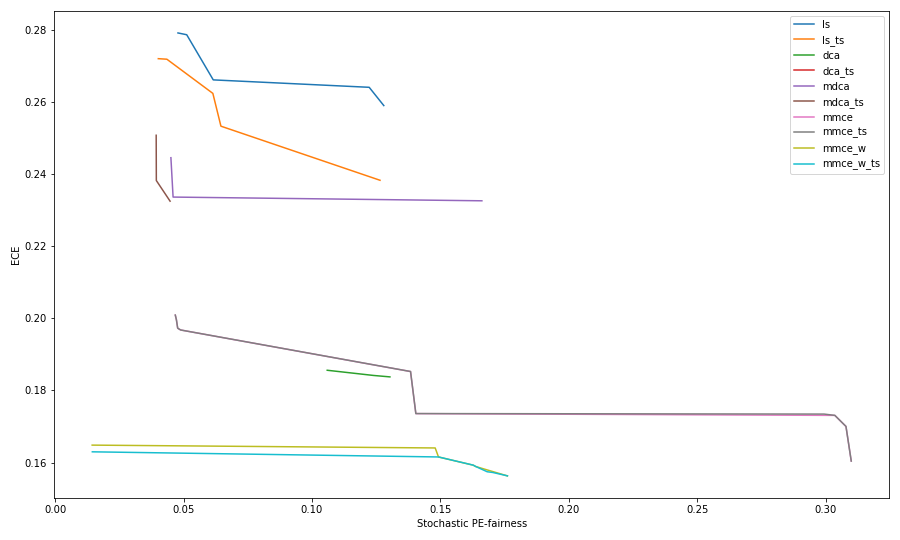}
\label{fig:arrhythmia_hybrid}
\end{minipage}
\end{figure}

\begin{figure}
\begin{minipage}{0.49\linewidth}
\centering
\caption{Pareto-optimal curve: Compas: Accuracy in (63.37\%, 68.37\%) with hybrid techniques}
\includegraphics[width=\linewidth]{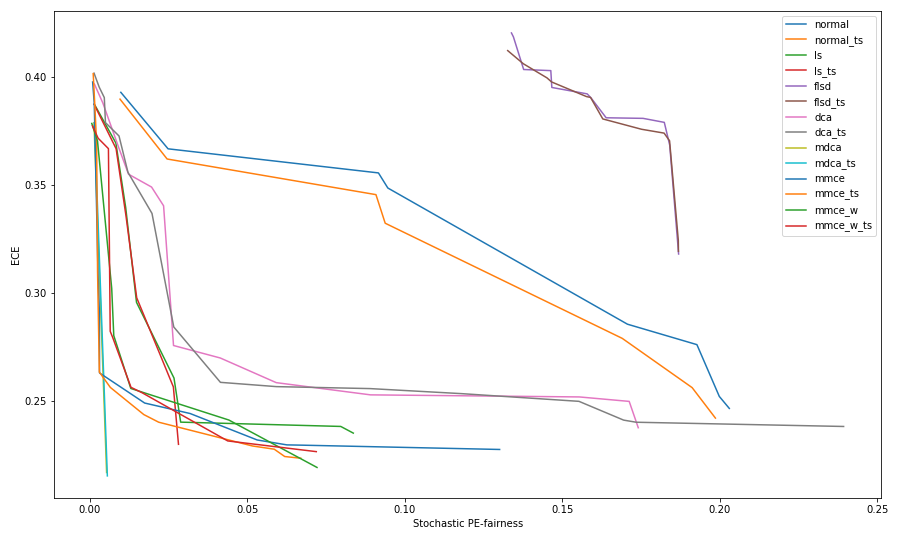}
\label{fig:compas_hybrid}
\end{minipage}
\hfill\vline\hfill
\begin{minipage}{0.49\linewidth}
\centering
\caption{Pareto-optimal curve: Drug: Accuracy in (78.47\%, 83.47\%) with hybrid techniques}
\includegraphics[width=\linewidth]{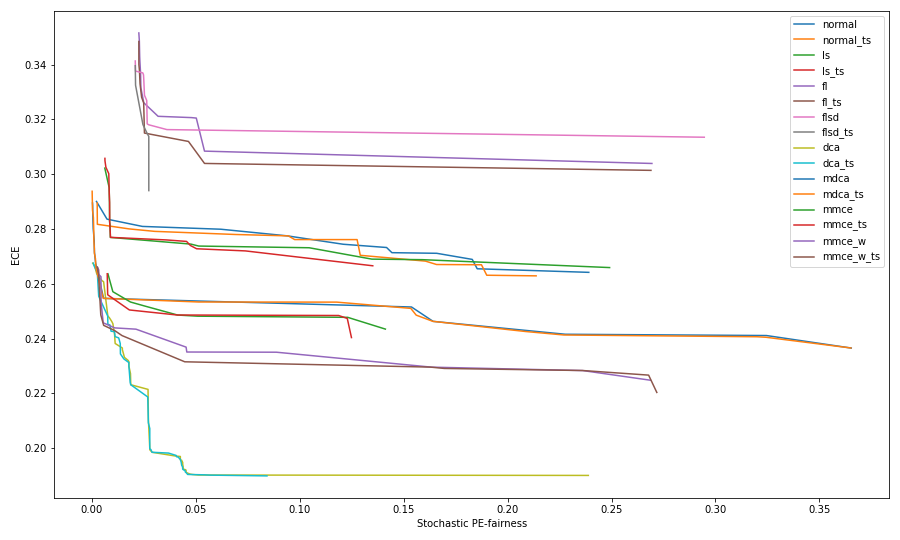}
\label{fig:drug_hybrid}
\end{minipage}
\end{figure}

\begin{figure}
\begin{minipage}{0.49\linewidth}
\centering
\caption{Pareto-optimal curve: German: Accuracy in (81.4\%, 86.4\%) with hybrid techniques}
\includegraphics[width=\linewidth]{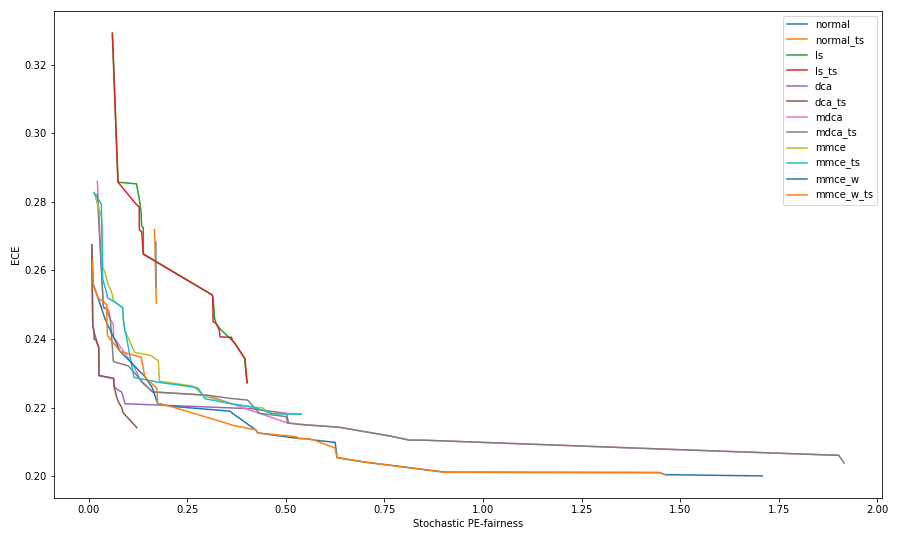}
\label{fig:german_hybrid}
\end{minipage}
\hfill\vline\hfill
\begin{minipage}{0.49\linewidth}
\centering
\caption{Pareto-optimal curve: Lawschool: Accuracy in (80.09\%, 85.09\%) with hybrid techniques}
\includegraphics[width=\linewidth]{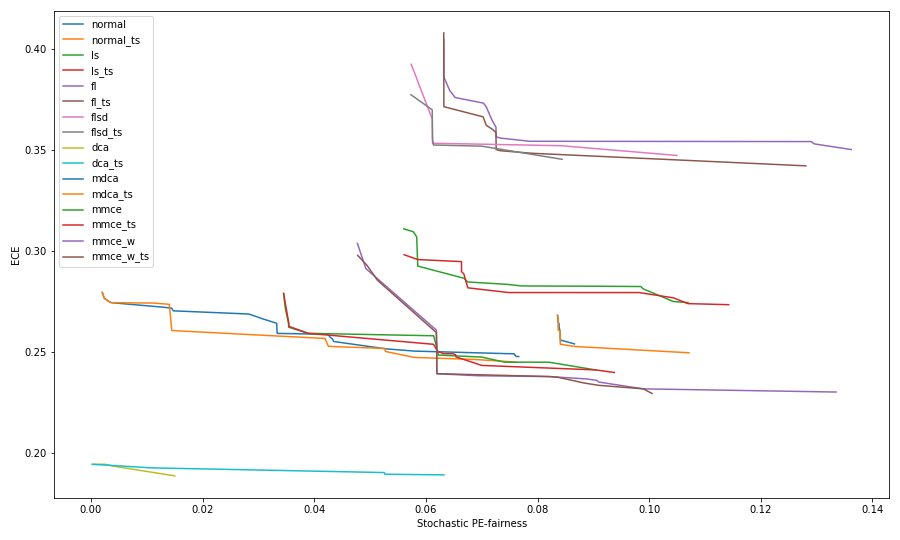}
\label{fig:lawschool_hybrid}
\end{minipage}
\end{figure}

\FloatBarrier
\subsection{Code snippets}
We write all our code in Tensorflow 2. The following code is used to compute the group-wise calibration loss for {\sf LS}, {\sf FL} and {\sf FLSD}.\\
\begin{lstlisting}
def groupwise_calibration_common_loss(
    calibration_loss: Callable[[tf.Tensor, tf.Tensor], tf.Tensor],
    rho: float,
    y: tf.Tensor,
    logits: tf.Tensor,
    a: tf.Tensor,
) -> tf.Tensor:
  y_maj = tf.gather(y, tf.where(a)[:, 0], axis=0)
  y_min = tf.gather(y, tf.where(1 - a)[:, 0], axis=0)
  logits_maj = tf.gather(logits, tf.where(a)[:, 0], axis=0)
  logits_min = tf.gather(logits, tf.where(1 - a)[:, 0], axis=0)
  return rho * calibration_loss(y_maj, logits_maj) + (1 - rho) * calibration_loss(
      y_min, logits_min
  )
\end{lstlisting}
The following code is used to compute group-wise calibration loss for {\sf DCA} and {\sf MDCA}.\\
\begin{lstlisting}
def groupwise_caliberation_separate(
    training_loss: tfk.losses.Loss,
    calibration_loss: Callable[[tf.Tensor, tf.Tensor], tf.Tensor],
    lbd: float,
    rho: float,
    y: tf.Tensor,
    logits: tf.Tensor,
    a: tf.Tensor,
) -> tf.Tensor:
  """
  Pass Sparse Categorical cross entropy with mean reduction as training loss.
  """
  y_maj = tf.gather(y, tf.where(a)[:, 0], axis=0)
  y_min = tf.gather(y, tf.where(1 - a)[:, 0], axis=0)
  logits_maj = tf.gather(logits, tf.where(a)[:, 0], axis=0)
  logits_min = tf.gather(logits, tf.where(1 - a)[:, 0], axis=0)
  return training_loss(y, logits) + lbd * (
      rho * calibration_loss(y_maj, logits_maj) + (1 - rho) * calibration_loss(y_min, logits_min)
  )
\end{lstlisting}
We implement group-wise versions of {\sf MMCE} and {\sf MMCE-W} separately. They are as shown below.\\
\begin{lstlisting}
def laplacian_kernel(r_1: tf.Tensor, r_2: tf.Tensor, gamma: float = 0.2) -> tf.Tensor:
  return tf.exp(-1.0 * tf.abs(r_1[tf.newaxis, :] - r_2[:, tf.newaxis]) / (2 * gamma))

def groupwise_mmce(
    training_loss: tfk.losses.Loss,
    lbd: float,
    rho: float,
    y: tf.Tensor,
    logits: tf.Tensor,
    a: tf.Tensor
) -> tf.Tensor:
  y_pred = tfk.activations.softmax(logits, axis=-1)
  c = tf.stop_gradient(tf.cast(
      tf.math.equal(y, tf.cast(tf.argmax(y_pred, axis=-1), tf.int32)), tf.float32))
  r = tf.reduce_max(y_pred, axis=-1)
  c_minus_r = c - r
  c_minus_r_maj = tf.gather(c_minus_r, tf.where(a)[:, 0], axis=0)
  c_minus_r_min = tf.gather(c_minus_r, tf.where(1 - a)[:, 0], axis=0)
  r_maj = tf.gather(r, tf.where(a)[:, 0], axis=0)
  r_min = tf.gather(r, tf.where(1 - a)[:, 0], axis=0)
  loss_1 = tf.reduce_mean((c_minus_r_maj[tf.newaxis, :] * c_minus_r_maj[:, tf.newaxis]) * laplacian_kernel(r_maj, r_maj))
  loss_2 = tf.reduce_mean((c_minus_r_min[tf.newaxis, :] * c_minus_r_min[:, tf.newaxis]) * laplacian_kernel(r_min, r_min))
  loss_3 = tf.reduce_mean((c_minus_r_maj[tf.newaxis, :] * c_minus_r_min[:, tf.newaxis]) * laplacian_kernel(r_maj, r_min))
  return training_loss(y, logits) + lbd * tf.sqrt(rho * rho * loss_1 + (1 - rho) * (1 - rho) * loss_2 + 2.0 * rho * (1 - rho) * loss_3)

def groupwise_mmce_w(
    training_loss: tfk.losses.Loss,
    lbd: float,
    rho: float,
    y: tf.Tensor,
    logits: tf.Tensor,
    a: tf.Tensor
) -> tf.Tensor:
  y_pred = tfk.activations.softmax(logits, axis=-1)
  c = tf.stop_gradient(tf.cast(
      tf.math.equal(y, tf.cast(tf.argmax(y_pred, axis=-1), tf.int32)), tf.int32))
  r = tf.reduce_max(y_pred, axis=-1)
  r_corr = tf.gather(r, tf.where(c)[:, 0], axis=0)
  r_wrong = tf.gather(r, tf.where(1 - c)[:, 0], axis=0)
  # Correct prediction loss classwise
  r_corr_maj = tf.gather(r_corr, tf.where(a)[:, 0], axis=0)
  r_corr_min = tf.gather(r_corr, tf.where(1 - a)[:, 0], axis=0)
  l1 = tf.reduce_mean(((1.0 - r_corr_maj)[tf.newaxis, :] * (1.0 - r_corr_maj)[:, tf.newaxis]) * laplacian_kernel(r_corr_maj, r_corr_maj))
  l2 = tf.reduce_mean(((1.0 - r_corr_min)[tf.newaxis, :] * (1.0 - r_corr_min)[:, tf.newaxis]) * laplacian_kernel(r_corr_min, r_corr_min))
  l3 = tf.reduce_mean(((1.0 - r_corr_maj)[tf.newaxis, :] * (1.0 - r_corr_min)[:, tf.newaxis]) * laplacian_kernel(r_corr_maj, r_corr_min))
  loss1 = rho * rho * l1 + (1 - rho) * (1 - rho) * l2 + 2.0 * rho * (1 - rho) * l3
  # Wrong prediction loss classwise
  r_wrong_maj = tf.gather(r_wrong, tf.where(a)[:, 0], axis=0)
  r_wrong_min = tf.gather(r_wrong, tf.where(1 - a)[:, 0], axis=0)
  l4 = tf.reduce_mean((r_wrong_maj[tf.newaxis, :] * r_wrong_maj[:, tf.newaxis]) * laplacian_kernel(r_wrong_maj, r_wrong_maj))
  l5 = tf.reduce_mean((r_wrong_min[tf.newaxis, :] * r_wrong_min[:, tf.newaxis]) * laplacian_kernel(r_wrong_min, r_wrong_min))
  l6 = tf.reduce_mean((r_wrong_maj[tf.newaxis, :] * r_wrong_min[:, tf.newaxis]) * laplacian_kernel(r_wrong_maj, r_wrong_min))
  loss2 = rho * rho * l4 + (1 - rho) * (1 - rho) * l5 + 2.0 * rho * (1 - rho) * l6
  # One correct and one incorrect classwise
  l7 = tf.reduce_mean(((1 - r_corr_maj)[tf.newaxis, :] * r_wrong_maj[:, tf.newaxis]) * laplacian_kernel(r_corr_maj, r_wrong_maj))
  l8 = tf.reduce_mean(((1 - r_corr_min)[tf.newaxis, :] * r_wrong_min[:, tf.newaxis]) * laplacian_kernel(r_corr_min, r_wrong_min))
  l9 = tf.reduce_mean(((1 - r_corr_maj)[tf.newaxis, :] * r_wrong_min[:, tf.newaxis]) * laplacian_kernel(r_corr_maj, r_wrong_min))
  l10 = tf.reduce_mean(((1 - r_corr_min)[tf.newaxis, :] * r_wrong_maj[:, tf.newaxis]) * laplacian_kernel(r_corr_min, r_wrong_maj))
  loss3 = rho * rho * l7 + (1 - rho) * (1 - rho) * l8 + rho * (1 - rho) * l9 + rho * (1 - rho) * l10
  return training_loss(y, logits) + lbd * tf.sqrt(loss1 + loss2 - 2.0*loss3)
\end{lstlisting}

%% file: table_hard_fairness.tex
\centering
\scriptsize
\captionof{table}{Percentage change in \textit{deterministic} {\sf PE}-fairness, {\sf ECE} and Accuracy represented as \%fair, \%calib. and \%acc. respectively when optimizing for fairness.}
\label{tab:hard_fair}
\begin{tabular}{lcccc} \toprule
\textbf{Dataset} & \textbf{Best technique} & \textbf{\%fair} & \textbf{\%calib.} & \textbf{\%acc.} \\ \midrule
Adult            & MMCE-W                  & 84.4            & -25.65            & -5.52           \\
Arrhythmia       & MMCE                    & 60.26           & -33.07            & -6.02           \\
Communities      & FL                      & 71.07           & -133.67           & -6.7            \\
Compas           & DCA                     & 73.31           & -26.65            & -12.2           \\
Drug             & MMCE                    & 92.33           & -57.51            & -19.27          \\
German           & FLSD                    & 92.2            & -30.58            & 1.11            \\
Lawschool        & MMCE                    & 75.32           & -26.86            & -21.08         \\ \bottomrule
\end{tabular}

%% file: table_fair_hybrid.tex
\centering
\scriptsize
\captionof{table}{Percentage change in \textit{stochastic} {sf PE}-fairness, {\sf ECE} and Accuracy represented as \%fair, \%calib. and \%acc. respectively when optimizing for fairness with \textit{hybrid} techniques.}
\label{tab:fair_hybrid}
\begin{tabular}{lcccc} \toprule
\textbf{Dataset} & \textbf{Best technique} & \textbf{\%fair} & \textbf{\%calib.} & \textbf{\%acc.} \\ \midrule
Adult            & MMCE                    & 96.28           & -18.39            & -3.22           \\
Arrhythmia       & MMCE-W                  & 86.25           & -29.69            & -9.36           \\
Communities      & MMCE                    & 87.39           & -26.27            & -2.54           \\
Compas           & MMCE                    & 91.47           & -28.41            & -2.5            \\
Drug             & MMCE                    & 94.53           & -5.43             & -4.17           \\
German           & MDCA                    & 97.14           & -41.24            & -11.33          \\
Lawschool        & DCA                     & 97.46           & -109.16           & -46.35          \\ \bottomrule
\end{tabular}

%% file: table_calib_hybrid.tex
\centering
\scriptsize
\captionof{table}{Percentage change in \textit{stochastic} {\sf PE}-fairness, {\sf ECE} and Accuracy represented as \%fair, \%calib. and \%acc. respectively when optimizing for calibration using \textit{hybrid} techniques.}
\label{tab:calib_hybrid}
\begin{tabular}{lcccc} \toprule
\textbf{Dataset} & \textbf{Best technique} & \textbf{\%fair} & \textbf{\%calib.} & \textbf{\%acc.} \\ \midrule
Adult            & DCA                     & -612.64         & 16.14             & 2.37            \\
Arrhythmia       & MMCE-W                  & -543.06         & 19.01             & 5.44            \\
Communities      & DCA                     & -2956.63        & 15.46             & 0.29            \\
Compas           & MMCE                    & -386.65         & 24.86             & -2.78           \\
Drug             & DCA                     & -2928.73        & 18.66             & -1.29           \\
German           & DCA                     & -295.55         & 13.35             & 0.48            \\
Lawschool        & DCA                     & -56.61          & 20.89             & -0.73           \\ \bottomrule
\end{tabular}